\documentclass[12pt,a4paper]{amsart}

\usepackage[T1]{fontenc}
\usepackage{lmodern}
\usepackage{mathtools,amssymb}
\usepackage{booktabs}
\usepackage{float}
\usepackage{needspace}
\usepackage[margin=2.9cm]{geometry}
\usepackage[numbers,sort&compress]{natbib}
\usepackage{hyperref}
\usepackage[nameinlink,noabbrev]{cleveref}

\hypersetup{
  hidelinks,
  pdftitle={Continuity-Free Near-Minimax Leading-Order Regret for CVaR-UCBVI},
  pdfauthor={Yuanlong Chen}
}

\numberwithin{equation}{section}

\theoremstyle{plain}
\newtheorem{theorem}{Theorem}[section]
\newtheorem{lemma}[theorem]{Lemma}
\newtheorem{proposition}[theorem]{Proposition}
\newtheorem{corollary}[theorem]{Corollary}
\theoremstyle{definition}

\newtheorem{example}[theorem]{Example}

\newcommand{\cvar}{\operatorname{CVaR}}
\newcommand{\Var}{\operatorname{Var}}
\newcommand{\E}{\mathbb{E}}
\newcommand{\Reg}{\operatorname{Regret}}
\newcommand{\cF}{\mathcal{F}}
\newcommand{\Vlow}{\widehat V^{\downarrow}}
\newcommand{\Vup}{\widehat V^{\uparrow}}
\newcommand{\rhop}{\widehat\rho}
\newcommand{\bhat}{\widehat b}
\newcommand{\BON}{\operatorname{BON}}
\newcommand{\ind}{\mathbf 1}

\title[Continuity-free CVaR-UCBVI regret]{Continuity-Free Near-Minimax
Leading-Order Regret for CVaR-UCBVI}
\author[Y. Chen]{Yuanlong Chen}
\begin{document}

\begin{abstract}
For finite-horizon tabular CVaR reinforcement learning, \citet{wang2023}
prove a $\widetilde O(\tau^{-1}\sqrt{SAK})$ leading regret bound for arbitrary
normalized return laws and the sharper $\widetilde O(\sqrt{SAK/\tau})$ rate
under a density lower bound. We show that the same Bernstein CVaR-UCBVI
algorithm attains the sharper rate without continuity assumptions. The key is
a selected-budget self-bound: the conditional variance of the episode
shortfall is at most $\tau$ plus the value-estimation width. Substitution into
the original Bernstein decomposition yields, with high probability,
$\widetilde O(\sqrt{SAK/\tau}+(SAH K^{1/4}+S^2AH)/\tau)$ regret for arbitrary
normalized return laws---atomic, mixed, or continuous. The $\tau^{-1/2}$
leading term matches the expected-regret minimax lower bound of
\citet{wang2023} up to logarithmic factors. Thus Bernstein CVaR-UCBVI is
minimax-optimal over the full return-law class in the leading-order regime;
the lower-order terms retain their $\tau^{-1}$ dependence.
\end{abstract}

\maketitle

\section{Introduction}
\label{sec:introduction}

Conditional value-at-risk (CVaR) evaluates the lower tail of a return
distribution and is a standard objective for risk-sensitive reinforcement
learning. In finite-horizon tabular Markov decision processes (MDPs) with
normalized returns, \citet{wang2023} establish a
$\sqrt{SAK/\tau}$ expected-regret lower bound on a balanced-tree hard family
and introduce CVaR-UCBVI, an augmented-state variant of
UCBVI~\citep{azar2017}. Their general upper bound has leading term
$\widetilde O(\tau^{-1}\sqrt{SAK})$. Under their Assumption~5.4, which
requires every relevant return law to have a density uniformly bounded away
from zero, the leading term improves to
$\widetilde O(\sqrt{SAK/\tau})$. The paper explicitly leaves removal of this
regularity assumption as an open problem.

This note removes the continuity assumption from the analysis of Bernstein
CVaR-UCBVI. The relevant term in Appendix~G of \citet{wang2023} is the
cumulative conditional variance
\[
  \sum_{k=1}^K
  \Var\!\left(
    \left(\bhat_k-\sum_{h=1}^H r_{h,k}\right)^+
    \middle|\cF_k
  \right).
\]
For general return laws, their analysis bounds every summand by $1$; under
Assumption~5.4, it instead uses continuity to compare $\bhat_k$ with a true
$\tau$-quantile. Our argument shows that the \emph{mean shortfall} at
$\bhat_k$ is at most $\tau$ plus the value-estimation width. Since a
$[0,1]$-valued random variable has variance at most its mean, the desired
scaling follows without quantile identification.

\Needspace{8\baselineskip}
\paragraph{Contribution.}
We establish a selected-budget self-bound that controls each conditional
shortfall variance by $\tau$ plus the value-estimation width. Combining this
bound with the Bernstein regret decomposition of \citet{wang2023} yields
Theorem~\ref{thm:main} for arbitrary normalized return laws, without modifying
the algorithm or its confidence event. The theorem covers atomic, mixed, and
continuous return laws. When the displayed lower-order terms are negligible,
its leading term matches the expected-regret minimax lower bound of
\citet[Corollary~3.2]{wang2023} up to logarithmic factors, establishing
leading-order minimax optimality over the full normalized-return class.
\Cref{app:reduction} restates the required auxiliary bounds in our notation.

\begin{table}[H]
  \caption{Leading square-root-in-$K$ term for Bernstein CVaR-UCBVI.
  Lower-order terms and logarithms are omitted.}
  \label{tab:comparison}
  \centering
  \begin{tabular}{@{}lll@{}}
    \toprule
    Result & Return-law condition & Leading term \\
    \midrule
    \citet[Theorem~5.3]{wang2023}
      & arbitrary normalized laws
      & $\tau^{-1}\sqrt{SAK}$ \\
    \citet[Theorem~5.5]{wang2023}
      & density at least $p_{\min}$
      & $\sqrt{SAK/\tau}$ \\
    Theorem~\ref{thm:main}
      & arbitrary normalized laws
      & $\sqrt{SAK/\tau}$ \\
    \bottomrule
  \end{tabular}
\end{table}

\subsection{Related work}

\citet{bastani2022} give regret guarantees for a broad class of static,
trajectory-level risk criteria. \citet{wang2023} specialize to precommitted
CVaR and obtain the bounds compared in \cref{tab:comparison}.
\citet{wang2025} reduce optimized certainty equivalents to risk-neutral RL;
their sharp CVaR specialization still assumes continuously distributed returns
with a density lower bound
\citep[Assumption~D.11 and Theorem~D.12]{wang2025}.
\citet{ni2024} study reward-free exploration for CVaR, which is a different
data-collection protocol and a PAC rather than online-regret objective.
\citet{liang2024} analyze dynamic, recursively composed risk measures, which
are time-consistent but differ from the static CVaR of the full trajectory
studied here. This paper addresses the complementary question of removing the
continuity condition from the regret analysis of Bernstein CVaR-UCBVI for
static trajectory-level CVaR.

\section{Problem setting and algorithm}
\label{sec:setup}

\subsection{Model and objective}

Consider an episodic tabular MDP with state space $\mathcal S$, action space
$\mathcal A$, sizes $S=|\mathcal S|$ and $A=|\mathcal A|$, horizon $H$, fixed
initial state $s_1$, and $K$ episodes. The transition kernel $P^*$ is unknown
and the conditional reward laws are known, as in \citet{wang2023}. Rewards are
nonnegative. Let $\mathcal H_h=(s_t,a_t,r_t)_{t<h}$ denote the history before
stage $h$. Conditional on $(s_h,a_h)$, the reward and next state are drawn
independently as
\[
  r_h\sim R(s_h,a_h),
  \qquad
  s_{h+1}\sim P^*(\cdot\mid s_h,a_h).
\]
The return of every history-dependent policy $\pi$ is almost surely
normalized:
\[
  R(\pi):=\sum_{h=1}^H r_h\in[0,1].
\]
For $\tau\in(0,1]$, the lower-tail CVaR of a random return $X\in[0,1]$ has the
representation~\citep{rockafellar2000}
\begin{equation}
  \cvar_\tau(X)
  =\max_{b\in[0,1]}
  \left\{b-\tau^{-1}\E[(b-X)^+]\right\}.
  \label{eq:cvar}
\end{equation}
Let $\cvar_\tau^*=\sup_\pi\cvar_\tau(R(\pi))$. A pair $(\rho,b)$ induces the
history-dependent policy
\[
  \pi_h^{\rho,b}(s_h,\mathcal H_h)
  =\rho_h\!\left(s_h,b-\sum_{t<h}r_t\right).
\]
Write its return as $R(\rho,b)$. The regret of the pairs
$(\rhop_k,\bhat_k)$ selected by the algorithm is
\begin{equation}
  \Reg^{\mathrm{RL}}_\tau(K)
  :=\sum_{k=1}^K
  \left[
    \cvar_\tau^*-
    \cvar_\tau\!\left(R(\rhop_k,\bhat_k)\right)
  \right].
  \label{eq:regret}
\end{equation}

The augmented state is $(s,b)$. Its initial budget lies in $[0,1]$ and evolves
as $b_{h+1}=b_h-r_h$; residual budgets can therefore be negative. For an
augmented policy $\rho$, define its shortfall value at every reachable budget
by
\begin{equation}
  V_h^\rho(s,b)
  :=\E_\rho\!\left[
    \left(b-\sum_{t=h}^H r_t\right)^+
    \middle|s_h=s,\ b_h=b
  \right],
  \qquad
  V_{H+1}^\rho(s,b)=b^+,
  \label{eq:value}
\end{equation}
and let $V_h^*(s,b)=\inf_\rho V_h^\rho(s,b)$.
The augmented-state optimality identity of \citet[Theorem~5.1]{wang2023}
gives, for every $b\in[0,1]$,
\begin{equation}
  V_1^*(s_1,b)
  =\inf_{\pi\ \text{history-dependent}}
  \E_\pi[(b-R(\pi))^+],
  \qquad
  \cvar_\tau^*
  =\max_{b\in[0,1]}\{b-\tau^{-1}V_1^*(s_1,b)\}.
  \label{eq:augmented-optimality}
\end{equation}

\subsection{Bernstein CVaR-UCBVI}

We summarize the exact-budget version of Algorithm~2 of \citet{wang2023} so
that the quantities used below are defined within the note. Let $\cF_k$ be the
complete history before episode $k$, and define the pre-episode counts and
empirical transition vector by
\begin{align}
  N_k(s,a,s')
  &:=\sum_{i<k}\sum_{h=1}^H
  \ind\{(s_{h,i},a_{h,i},s_{h+1,i})=(s,a,s')\},
  \notag\\
  N_k(s,a)&:=1\vee\sum_{s'}N_k(s,a,s'),
  &
  \widehat P_k(s'\mid s,a)&:=\frac{N_k(s,a,s')}{N_k(s,a)}.
  \label{eq:empirical-transition}
\end{align}
Set $\Vlow_{H+1,k}(s,b)=\Vup_{H+1,k}(s,b)=b^+$. For
$h=H,H-1,\ldots,1$, the algorithm applies the backward updates
\begin{align}
  \widehat U^{\downarrow}_{h,k}(s,b,a)
  &:={}
  \sum_{s'}\widehat P_k(s'\mid s,a)
  \E_{r\sim R(s,a)}[\Vlow_{h+1,k}(s',b-r)]
  -\BON^{\mathrm{BERN}}_{h,k}(s,b,a),
  \notag\\
  \rhop_{h,k}(s,b)
  &\in\arg\min_{a\in\mathcal A}
  \widehat U^{\downarrow}_{h,k}(s,b,a),
  \notag\\
  \Vlow_{h,k}(s,b)
  &:=\max\{\widehat U^{\downarrow}_{h,k}
  (s,b,\rhop_{h,k}(s,b)),0\},
  \label{eq:pessimistic-update}\\
  \widehat U^{\uparrow}_{h,k}(s,b,a)
  &:={}
  \sum_{s'}\widehat P_k(s'\mid s,a)
  \E_{r\sim R(s,a)}[\Vup_{h+1,k}(s',b-r)]
  +\BON^{\mathrm{BERN}}_{h,k}(s,b,a),
  \notag\\
  \Vup_{h,k}(s,b)
  &:=\min\{\widehat U^{\uparrow}_{h,k}
  (s,b,\rhop_{h,k}(s,b)),1\}.
  \label{eq:optimistic-update}
\end{align}
The recursion is evaluated on all reachable residual budgets, with a fixed
measurable rule for resolving ties. The Bernstein bonus is given explicitly
in \cref{eq:bonus-exact}. After the backward pass, the algorithm selects
\begin{equation}
  \bhat_k\in\arg\max_{b\in[0,1]}
  \widehat f_k(b),
  \qquad
  \widehat f_k(b):=b-\tau^{-1}\Vlow_{1,k}(s_1,b).
  \label{eq:selected-budget}
\end{equation}
As in the exact-budget formulation of \citet[Algorithm~2]{wang2023}, we use
an exact-optimization oracle: we assume that \cref{eq:selected-budget} admits
an $\cF_k$-measurable maximizer and fix one such selector. Attainment is
automatic for the finite-grid implementation in \cref{cor:discrete}. All
quantities in \cref{eq:empirical-transition,eq:selected-budget} are then
$\cF_k$-measurable. The episode starts from $b_{1,k}=\bhat_k$ and, for
$h=1,\ldots,H$, follows
\[
  \begin{aligned}
    a_{h,k}&=\rhop_{h,k}(s_{h,k},b_{h,k}),
    &r_{h,k}&\sim R(s_{h,k},a_{h,k}),\\
    s_{h+1,k}&\sim P^*(\cdot\mid s_{h,k},a_{h,k}),
    &b_{h+1,k}&=b_{h,k}-r_{h,k}.
  \end{aligned}
\]
Thus the algorithm interacts only with the original MDP while maintaining the
scalar residual budget. Expectations or variances indexed by $(\rho,b)$ refer
to this true-MDP rollout starting from $(s_1,b)$.

\subsection{Auxiliary regret decomposition}

Define
\begin{align}
  Y_k
  &:=\left(\bhat_k-\sum_{h=1}^H r_{h,k}\right)^+,
  &
  W_k
  &:=V_1^{\rhop_k}(s_1,\bhat_k)-\Vlow_{1,k}(s_1,\bhat_k),
  \label{eq:Y-W}\\
  Q_K
  &:=\sum_{k=1}^K\Var(Y_k\mid\cF_k).
  \label{eq:Q}
\end{align}
In particular,
\begin{equation}
  \E[Y_k\mid\cF_k]=V_1^{\rhop_k}(s_1,\bhat_k).
  \label{eq:mean-value}
\end{equation}

For later use, let $\xi_k$ denote the transition correction associated with
the Bernstein bonus. Its formula and the bonus formula are recorded in
\cref{app:reduction}. For each realized stage, set
\[
  g_{h,k}(s')
  :=\E_{r\sim R(s_{h,k},a_{h,k})}
  V_{h+1}^{\rhop_k}(s',b_{h,k}-r).
\]
Define
\begin{align}
  B_k
  &:=
  \sum_{h=1}^H
  \E_{\rhop_k,\bhat_k}\!\left[
    2\BON^{\mathrm{BERN}}_{h,k}(s_h,b_h,a_h)
    +\xi_k(s_h,a_h)
    \middle|\cF_k
  \right],
  \label{eq:Bk}\\
  \mathfrak D_K
  &:=
  \sum_{k=1}^K\sum_{h=1}^H
  \sqrt{
    \frac{
      \Var_{s'\sim P^*(\cdot\mid s_{h,k},a_{h,k})}
      \!\left(g_{h,k}(s')\right)
    }{N_k(s_{h,k},a_{h,k})}
  }.
  \label{eq:D-def}
\end{align}
Each variance in \cref{eq:D-def} is over a fresh next-state draw after the
realized trajectory variables have been substituted. Thus
$\mathfrak D_K$ is trajectory-random; no conditioning on $\cF_k$ is intended
inside these one-step variances.
Set
\[
  \ell_\delta:=\log(HSAK/\delta),
  \qquad L:=1\vee\ell_\delta.
\]
The algorithm uses $\ell_\delta$ in its bonus; replacing it by the larger $L$
below only enlarges the published upper bounds.

\begin{proposition}[Auxiliary bounds from \citet{wang2023}]
\label{prop:published}
For the estimates generated by \cref{eq:pessimistic-update,eq:optimistic-update}
with the bonus in \cref{eq:bonus-exact}, the following bounds hold on a common
event with probability at least $1-\delta$. In particular,
\cref{eq:pessimism} holds for every $h,k$ and uniformly over its state and
budget arguments, while \cref{eq:simulation} holds for every episode $k$:
\begin{align}
  0\le \Vlow_{h,k}(s,b)&\le V_h^*(s,b),
  \label{eq:pessimism}\\
  0\le W_k&\le \mathrm e B_k,
  \label{eq:simulation}\\
  \sum_{k=1}^K B_k
  &\le
  6HL+12\sqrt{SAHKL}+6S^2AHL^2,
  \label{eq:bonus-sum}\\
  \mathfrak D_K
  &\le \sqrt{SAL(2HL+2Q_K)},
  \label{eq:variance-reduction}\\
  \Reg^{\mathrm{RL}}_\tau(K)
  &\le
  6\mathrm e\,\tau^{-1}\sqrt L\,\mathfrak D_K
  +54\mathrm e\,\tau^{-1}SAH K^{1/4}L^2
  +74\mathrm e\,\tau^{-1}S^2AHL^3.
  \label{eq:regret-reduction}
\end{align}
\end{proposition}

Proposition~\ref{prop:published} follows from Appendix~G of \citet{wang2023};
\cref{app:reduction} restates the relevant derivation in our notation.
Uniformity over $b$ permits evaluation at the data-dependent budget $\bhat_k$.
The remainder of the proof uses the analysis only through the
inequalities in Proposition~\ref{prop:published}.

\section{A self-bounding shortfall}
\label{sec:self-bound}

The main new estimate is the following lemma.

\begin{lemma}[Selected-budget self-bound]
\label{lem:self-bound}
On the event in Proposition~\ref{prop:published}, for every episode $k$,
\begin{equation}
  \Vlow_{1,k}(s_1,\bhat_k)
  \le \tau\bhat_k\le\tau,
  \qquad
  \Var(Y_k\mid\cF_k)
  \le\E[Y_k\mid\cF_k]
  \le\tau+W_k.
  \label{eq:self-bound}
\end{equation}
\end{lemma}

\begin{proof}
Nonnegative rewards imply
$V_1^*(s_1,0)=\inf_\rho\E_\rho[(-\sum_h r_h)^+]=0$. Pessimism and clipping
therefore give
\[
  0\le\Vlow_{1,k}(s_1,0)\le V_1^*(s_1,0)=0,
\]
so $\widehat f_k(0)=0$. Because $0$ is feasible in
\cref{eq:selected-budget}, optimality of $\bhat_k$ yields
\[
  0\le\widehat f_k(\bhat_k)
  =\bhat_k-\tau^{-1}\Vlow_{1,k}(s_1,\bhat_k).
\]
Thus
$\Vlow_{1,k}(s_1,\bhat_k)\le\tau\bhat_k\le\tau$.

Next, $0\le Y_k\le\bhat_k\le1$, hence $Y_k^2\le Y_k$ and
\[
  \Var(Y_k\mid\cF_k)
  \le\E[Y_k^2\mid\cF_k]
  \le\E[Y_k\mid\cF_k].
\]
Finally, by \cref{eq:mean-value,eq:Y-W},
\[
  \E[Y_k\mid\cF_k]
  =\Vlow_{1,k}(s_1,\bhat_k)+W_k
  \le\tau+W_k.
\]
Also $W_k\ge0$ because
$\Vlow_{1,k}\le V_1^*\le V_1^{\rhop_k}$.
\end{proof}

The lemma controls a lower partial moment rather than a CDF value. An atom at
the selected threshold may make
$\Pr(\sum_h r_{h,k}\le\bhat_k\mid\cF_k)$ large, but its mass at equality
contributes zero to $(\bhat_k-\sum_h r_{h,k})^+$. This is why no quantile
identification or anti-concentration is needed.

\begin{example}[An atomic Bernoulli return law]
\label{ex:atomic}
Let $H=1$, $R\in\{0,1\}$, and $\Pr(R=0)=\tau$. This law has no Lebesgue
density, and its population CVaR objective is
\[
  b-\tau^{-1}\E[(b-R)^+]
  =b-\tau^{-1}(\tau b)=0
  \qquad (b\in[0,1]).
\]
The maximizing budget is therefore non-unique, so an argument based on a
unique selected quantile is unavailable. Nevertheless, at the extreme choice
$b=1$, $Y=(1-R)^+=\ind\{R=0\}$ and
$\Var(Y)=\tau(1-\tau)\le\tau$. Thus the self-bound remains valid despite the
non-unique maximizing budget.
\end{example}

\section{Regret analysis}
\label{sec:regret}

The self-bound yields the following distribution-free control of cumulative
trajectory variance.

\begin{proposition}[Cumulative trajectory variance]
\label{prop:Q}
On the event in Proposition~\ref{prop:published},
\begin{equation}
  Q_K
  \le K\tau
  +\mathrm e\left(
    6HL+12\sqrt{SAHKL}+6S^2AHL^2
  \right).
  \label{eq:Q-bound}
\end{equation}
\end{proposition}

\begin{proof}
By the definition of $Q_K$ in \cref{eq:Q} and the episodewise bound in
\cref{eq:self-bound},
\[
  Q_K
  =\sum_{k=1}^K\Var(Y_k\mid\cF_k)
  \le\sum_{k=1}^K(\tau+W_k)
  =K\tau+\sum_{k=1}^K W_k.
\]
On the same event, \cref{eq:simulation} gives
$W_k\le\mathrm e B_k$ episodewise, and \cref{eq:bonus-sum} controls the
cumulative bonuses. Hence,
\[
  Q_K
  \le K\tau+\mathrm e\sum_{k=1}^K B_k
  \le K\tau
  +\mathrm e\left(
    6HL+12\sqrt{SAHKL}+6S^2AHL^2
  \right),
\]
which proves \cref{eq:Q-bound}.
\end{proof}

\begin{theorem}[Continuity-free leading-order regret]
\label{thm:main}
Fix $\tau\in(0,1]$ and $\delta\in(0,1)$. Under the tabular, known-reward,
nonnegative, normalized-return model of \cref{sec:setup}, run the exact-budget
version of Algorithm~2 of \citet{wang2023} with its Bernstein bonus. With
probability at least $1-\delta$, without Assumption~5.4 of that paper,
\begin{equation}
  \Reg^{\mathrm{RL}}_\tau(K)
  \le
  25L\sqrt{\frac{SAK}{\tau}}
  +\frac{300}{\tau}SAH K^{1/4}L^2
  +\frac{400}{\tau}S^2AHL^3.
  \label{eq:main}
\end{equation}
Consequently,
\begin{equation}
  \Reg^{\mathrm{RL}}_\tau(K)
  =\widetilde O\!\left(
    \sqrt{\frac{SAK}{\tau}}
    +\frac{SAH K^{1/4}+S^2AH}{\tau}
  \right)
  \label{eq:soft-main}
\end{equation}
for every MDP satisfying the stated model, regardless of whether its induced
return laws are discrete, mixed, or continuous.
\end{theorem}

\begin{proof}
Substitute \cref{eq:Q-bound} into \cref{eq:variance-reduction}:
\begin{align*}
  \mathfrak D_K^2
  \le SAL\bigl(&2K\tau+(2+12\mathrm e)HL
  +24\mathrm e\sqrt{SAHKL}
  +12\mathrm e S^2AHL^2\bigr).
\end{align*}
Using $\sqrt{x_1+\cdots+x_m}\le\sum_i\sqrt{x_i}$ and
$\mathrm e<2.72$ gives
\begin{align}
  \mathfrak D_K
  \le{}&\sqrt{2\tau SAKL}
  +\sqrt{35SAH}\,L
  +\sqrt{66}(SA)^{3/4}H^{1/4}K^{1/4}L^{3/4}
  \notag\\
  &+\sqrt{33}S^{3/2}AH^{1/2}L^{3/2}.
  \label{eq:D-expanded}
\end{align}
Insert \cref{eq:D-expanded} into \cref{eq:regret-reduction}. The first term
is at most
\[
  6\mathrm e\,\tau^{-1}\sqrt L
  \sqrt{2\tau SAKL}
  \le25L\sqrt{SAK/\tau}.
\]
Since $S,A,H,L\ge1$, the third term in
\cref{eq:D-expanded} is at most
$140\tau^{-1}SAHK^{1/4}L^2$; together with the existing
$54\mathrm e\le149$ coefficient, it is bounded by the second term of
\cref{eq:main}. The second and fourth terms in \cref{eq:D-expanded} contribute
at most $103\tau^{-1}S^2AHL^3$ and
$95\tau^{-1}S^2AHL^3$, respectively. Adding the remaining
$74\mathrm e\le202$ coefficient gives at most
$(103+95+202)\tau^{-1}S^2AHL^3=400\tau^{-1}S^2AHL^3$. This proves
\cref{eq:main}; \cref{eq:soft-main} follows.
\end{proof}

\begin{corollary}[Discretized implementation]
\label{cor:discrete}
For an integer $m\ge2$, set $\eta:=m^{-1}$ and let
\[
  \phi_\eta(r):=\min\{1,\eta\lceil r/\eta\rceil\},
  \qquad
  \mathcal B_\eta:=\{j\eta:0\le j\le m\}.
\]
Implement Algorithm~2 by replacing every reward with $\phi_\eta(r)$,
evaluating its backward recursion on the finite residual-budget set reachable
from $\mathcal B_\eta$ under the rounded rewards, and selecting
\begin{equation}
  \bhat_k\in\arg\max_{b\in\mathcal B_\eta}\widehat f_k(b).
  \label{eq:grid-budget}
\end{equation}
Under the assumptions of \cref{thm:main}, the regret of the resulting sequence
of policies in the original MDP is bounded by the right-hand side of
\cref{eq:main} plus
\[
  \frac{KH\eta}{\tau}.
\]
For any requested mesh width $\bar\eta\in(0,1)$, taking
$m=\lceil\bar\eta^{-1}\rceil$ gives $\eta\le\bar\eta$. No continuity
assumption is needed for either the original or discretized return law.
\end{corollary}

\begin{proof}
Write $M_\eta=\operatorname{disc}(M)$ and, for a nonnegative random variable
$X$, define the budget-restricted functional
\[
  \cvar_\tau^{[0,1]}(X)
  :=\max_{b\in[0,1]}
  \left\{b-\tau^{-1}\E[(b-X)^+]\right\}.
\]
It agrees with ordinary CVaR when $X\in[0,1]$.
In this proof, model subscripts distinguish returns and optimal shortfall
values in $M$ and $M_\eta$.

The rounded return in $M_\eta$ can be as large as $1+H\eta$, so $M_\eta$
need not satisfy the normalization assumption of \cref{thm:main}.
Nevertheless, the proof of that theorem applies without change to
$\cvar_\tau^{[0,1]}$: rewards remain nonnegative and the initial budget
remains in $[0,1]$, so every shortfall and shortfall value used in the
pessimism, simulation, concentration, and variance arguments remains in
$[0,1]$. Normalization is used only to identify the restricted functional
with ordinary CVaR. The augmented optimality identity also remains valid in
$M_\eta$, because it is a dynamic-programming identity for the shortfall
objective and does not require normalized total returns.

Let
\begin{align*}
  C_\eta^*
  &:=\sup_{b\in[0,1]}
  \{b-\tau^{-1}V_{1,\eta}^*(s_1,b)\},\\
  C_\eta^{\mathcal B}
  &:=\max_{b\in\mathcal B_\eta}
  \{b-\tau^{-1}V_{1,\eta}^*(s_1,b)\},\\
  C_{\eta,k}
  &:=\cvar_\tau^{[0,1]}
  \bigl(R_{M_\eta}(\rhop_k,\bhat_k)\bigr).
\end{align*}
The proof of \cref{thm:main} also applies with the grid benchmark
$C_\eta^{\mathcal B}$. Indeed, let $b_\eta^*$ maximize that benchmark and set
\[
  W_{\eta,k}
  :=V_{1,\eta}^{\rhop_k}(s_1,\bhat_k)
  -\Vlow_{1,k}(s_1,\bhat_k).
\]
Here $\widehat f_k$, $\Vlow_{1,k}$, and
$V_{1,\eta}^{\rhop_k}$ are all computed in the rounded MDP.
Pessimism, optimality in \cref{eq:grid-budget}, and evaluation of the
restricted CVaR objective at $\bhat_k$ give
\begin{align*}
  C_\eta^{\mathcal B}
  &\le b_\eta^*-\tau^{-1}\Vlow_{1,k}(s_1,b_\eta^*)
  \le \bhat_k-\tau^{-1}\Vlow_{1,k}(s_1,\bhat_k),\\
  C_{\eta,k}
  &\ge \bhat_k-\tau^{-1}
  V_{1,\eta}^{\rhop_k}(s_1,\bhat_k).
\end{align*}
Consequently,
$C_\eta^{\mathcal B}-C_{\eta,k}\le\tau^{-1}W_{\eta,k}$.
The grid contains $0$, so the selected-budget self-bound remains valid; all
subsequent simulation, bonus, concentration, and variance steps are
unchanged. Hence
\[
  \sum_{k=1}^K(C_\eta^{\mathcal B}-C_{\eta,k})
\]
is bounded by the right-hand side of \cref{eq:main}.

It remains to compare the grid and continuum benchmarks. For a fixed
history-dependent policy $\pi$, write $X_\pi=R_{M_\eta}(\pi)$ and
\[
  f_\pi(b):=b-\tau^{-1}\E[(b-X_\pi)^+].
\]
Every rounded stage reward belongs to $\mathcal B_\eta$. Therefore every
possible value of $X_\pi$ that lies in $[0,1]$ also belongs to
$\mathcal B_\eta$: it is either $1$ or an integer multiple of $\eta$.
Thus $f_\pi$ is concave and piecewise linear on $[0,1]$, with all
breakpoints and both endpoints in $\mathcal B_\eta$, and so
\[
  \max_{b\in[0,1]}f_\pi(b)
  =\max_{b\in\mathcal B_\eta}f_\pi(b).
\]
The history-dependent policy class in the augmented optimality identity does
not depend on $b$. Consequently,
\begin{align}
  C_\eta^*
  &=\sup_{b\in[0,1]}\sup_\pi f_\pi(b)
   =\sup_\pi\max_{b\in[0,1]}f_\pi(b)
  \notag\\
  &=\sup_\pi\max_{b\in\mathcal B_\eta}f_\pi(b)
   =\max_{b\in\mathcal B_\eta}\sup_\pi f_\pi(b)
   =C_\eta^{\mathcal B}.
  \label{eq:grid-benchmark-equality}
\end{align}

The adapted policy actually executed in $M$ is
\[
  \pi_{h,k}^\eta(s_h,\mathcal H_h)
  :=\rhop_{h,k}\!\left(
    s_h,\bhat_k-\sum_{t<h}\phi_\eta(r_t)
  \right).
\]
Thus its residual budget is updated with rounded rewards. Theorem~H.4 of
\citet{wang2023} gives
\[
  V_{1,\eta}^*(s_1,b)\le V_{1,M}^*(s_1,b)
  \quad(b\in[0,1]),
\]
and hence $C_\eta^*\ge\cvar_\tau^*(M)$. Theorem~H.3 of that paper gives
\[
  \cvar_\tau\bigl(R_M(\pi_k^\eta)\bigr)
  \ge
  \cvar_\tau\bigl(R_{M_\eta}(\rhop_k,\bhat_k)\bigr)
  -\frac{H\eta}{\tau}
  \ge C_{\eta,k}-\frac{H\eta}{\tau},
\]
where the last inequality uses that full CVaR dominates its
$[0,1]$-restricted counterpart.
Consequently, episodewise,
\[
  \cvar_\tau^*(M)-\cvar_\tau\bigl(R_M(\pi_k^\eta)\bigr)
  \le C_\eta^{\mathcal B}-C_{\eta,k}+\frac{H\eta}{\tau},
\]
where \cref{eq:grid-benchmark-equality} was used. Summing over $k$ proves the
claim. Thus Theorem~H.4 supplies the benchmark comparison, whereas
Theorem~H.3 supplies the selected-policy comparison.
\end{proof}

\section{Discussion}
\label{sec:discussion}

Theorem~\ref{thm:main} gives Bernstein CVaR-UCBVI a $\tau^{-1/2}$ leading
dependence for arbitrary normalized return laws. This term has the same
$\tau$- and $K$-dependence as the expected-regret minimax lower bound of
\citet[Corollary~3.2]{wang2023} on its atomic balanced-tree family, up to
logarithmic factors. Ignoring logarithms, the leading term dominates the two
remaining terms when
\[
  K\gtrsim \frac{S^2A^2H^4}{\tau^2}
  \qquad\text{and}\qquad
  K\gtrsim \frac{S^3AH^2}{\tau}.
\]
If $\tau$ decreases with $K$, the $\tau^{-1}$ lower-order terms may continue
to dominate; thus \cref{thm:main} does not establish uniform finite-sample
optimality over all joint parameter regimes.

The result applies to tabular MDPs with known reward laws, unknown
transitions, nonnegative normalized returns, and static CVaR of the full
trajectory. Corollary~\ref{cor:discrete} covers the grid-based implementation
of \citet{wang2023}, but does not alter its computational complexity.

\begin{proposition}[CVaR specialization of the OCE reduction]
\label{prop:oce}
Consider Algorithm~1 of \citet{wang2025}, specialized to CVaR and to
nonnegative normalized returns. Let $\cF_k$ contain the information available
immediately before episode $k$ is executed, and write
$\widehat V_k(b):=\widehat V_{1,k}(s_1,b)$ for the oracle's optimistic
initial value. The algorithm selects
\[
  \widehat b_k\in\arg\max_{b\in[0,1]}
  \{b+\widehat V_k(b)\}
\]
and an augmented policy $\pi^k$. These choices are $\cF_k$-measurable. Let
$Z(\pi^k,\widehat b_k)$ denote a fresh return induced by the selected pair,
and let $Z_k:=\sum_{h=1}^H r_{h,k}\in[0,1]$ be the realized return;
conditionally on $\cF_k$, these variables have the same law. Set
$Y_k^{\mathrm{OCE}}=(\widehat b_k-Z_k)^+$ and define
\[
  \Reg_{\mathrm{CVaR}_\tau}(K)
  :=\sum_{k=1}^K\left[
    \cvar_\tau^*-\cvar_\tau\bigl(Z(\pi^k,\widehat b_k)\bigr)
  \right].
\]
Suppose the oracle satisfies Definition~3.1 of \citet{wang2025} with exact
optimism,
\[
  \widehat V_k(b)\ge V_{\mathrm{aug}}^*(s_1,b)
  \qquad(b\in[0,1]),
\]
where $V_{\mathrm{aug}}^*$ is the optimal value in the CVaR augmented MDP,
and also satisfies Assumption~D.10 of that paper in the form
\[
  \Reg_{\mathrm{Opt}}(K)
  \le \sqrt{
    C_1\sum_{k=1}^K\Var(Y_k^{\mathrm{OCE}}\mid\cF_k)
  }+C_2.
\]
Then Assumption~D.11 is unnecessary, and
\begin{equation}
  \Reg_{\mathrm{CVaR}_\tau}(K)
  \le
  2\sqrt{\frac{C_1K}{\tau}}
  +\frac{C_1+2C_2}{\tau}.
  \label{eq:oce-extension}
\end{equation}
\end{proposition}

The proof is given in \cref{app:oce}. Proposition~\ref{prop:oce} concerns
only the CVaR specialization; it does not assert an analogous self-bound for
every OCE utility. In model classes with a matching CVaR lower bound, its
$\tau^{-1/2}$ leading dependence is minimax-optimal for atomic, mixed, and
continuous normalized return laws. The other oracle assumptions of
\citet{wang2025} remain unchanged.

In summary, direct mean-shortfall control gives the continuity-free leading term
$\widetilde O(\sqrt{SAK/\tau})$. It applies to atomic, mixed, and continuous
normalized return laws and matches the expected-regret minimax lower bound up
to logarithmic factors. The lower-order terms retain their $\tau^{-1}$ dependence.

\appendix

\section{Derivation of the auxiliary regret bounds}
\label{app:reduction}

For completeness, this appendix restates the ingredients from Appendix~G of
\citet{wang2023} in the notation used here and records the locations in their
paper. These bounds lead to Proposition~\ref{prop:published}; the
selected-budget argument enters subsequently through
Lemma~\ref{lem:self-bound}.

\subsection{Bernstein bonus and transition correction}

Let $\widehat P_k(\cdot\mid s,a)$ be the empirical transition law based on the
$N_k(s,a)$ pre-episode visits. With $\ell_\delta=\log(HSAK/\delta)$ and
$b'=b-r$, write
\[
  \widehat{\E}_{k,s,a}[g]
  :=\E_{\substack{s'\sim\widehat P_k(s,a)\\r\sim R(s,a)}}[g(s',r)].
\]
Equation~(6) of \citet{wang2023} defines
\begin{align}
  \BON^{\mathrm{BERN}}_{h,k}(s,b,a)
  :={}&
  \sqrt{
    \frac{2\ell_\delta}{N_k(s,a)}
    \Var_{s'\sim\widehat P_k(s,a)}
    \!\left(
      \E_{r\sim R(s,a)}
      \Vlow_{h+1,k}(s',b-r)
    \right)
  }
  \notag\\
  &+
  \sqrt{
    \frac{2\ell_\delta}{N_k(s,a)}
    \widehat{\E}_{k,s,a}
    \!\left[
      \left(
        \Vup_{h+1,k}(s',b')
        -\Vlow_{h+1,k}(s',b')
      \right)^2
    \right]
  }
  +\frac{\ell_\delta}{N_k(s,a)}.
  \label{eq:bonus-exact}
\end{align}
Lemma~G.2 introduces the transition correction
\begin{equation}
  \xi_k(s,a)
  :=\min\!\left\{1,\frac{2HS\ell_\delta}{N_k(s,a)}\right\}.
  \label{eq:xi}
\end{equation}
The confidence event constructed in Appendix~G.1 is uniform over
$b\in[0,1]$. On this event, Theorem~G.7 establishes the required pessimism
invariant, while Lemma~G.4 gives $W_k\le\mathrm e B_k$. These results yield
\cref{eq:pessimism,eq:simulation}.

\subsection{Cumulative width bound}

Write
\[
  G_{h,k}
  :=2\BON^{\mathrm{BERN}}_{h,k}(s_{h,k},b_{h,k},a_{h,k})
  +\xi_k(s_{h,k},a_{h,k}).
\]
The first Azuma-type bound in Appendix~G.1, labeled \emph{Azuma 1} there and
summed over $h$, gives
\[
  \sum_{k=1}^K B_k
  \le 6H\ell_\delta+2\sum_{h=1}^H\sum_{k=1}^K G_{h,k}.
\]
Equation~(15) of that appendix gives the variance-independent estimate
\[
  \sum_{h=1}^H\sum_{k=1}^K G_{h,k}
  \le 6\sqrt{SAHK\ell_\delta}+3S^2AH\ell_\delta^2.
\]
Using $L\ge\ell_\delta$ proves \cref{eq:bonus-sum}. Because this estimate does
not use the empirical-variance terms, it controls $\sum_k W_k$ independently
of the variance bound in the next subsection.

\subsection{Population variance and regret decomposition}

Equation~(19) of \citet{wang2023} corresponds to the term
$\mathfrak D_K$ defined in \cref{eq:D-def}. Cauchy--Schwarz and the
elliptical-potential count bound contribute the factor $SAL$. The paper's
\emph{Azuma 4} bound and the fact that joint variance dominates the variance
of a conditional mean then yield
\[
  \mathfrak D_K
  \le
  \sqrt{
    SAL\!\left(
      2HL+
      2\sum_{k=1}^K
      \Var_{\rhop_k,\bhat_k}\!\left(
        \left(\bhat_k-\sum_{h=1}^H r_h\right)^+
        \middle|\cF_k
      \right)
    \right)
  }.
\]
Lemma~G.9 identifies the inner trajectory variance by the law of total
variance; this is \cref{eq:variance-reduction}.

Finally, the display immediately before the paper's two alternative bounds on
Eq.~(19) collects the correction term, empirical-to-population switching
cost, and all $1/N_k$ terms into the bound
\[
  27SAHK^{1/4}L^2+34S^2AHL^3+3\sqrt L\,\mathfrak D_K.
\]
The regret decomposition following Eq.~(13) in the proof of Theorem~5.2 of
\citet[Appendix~G.3, pp.~29--30]{wang2023} gives, episodewise,
\[
  \cvar_\tau^*-
  \cvar_\tau\!\left(R(\rhop_k,\bhat_k)\right)
  \le \tau^{-1}W_k.
\]
It follows from pessimism, optimality of the selected budget, and evaluation
of the selected-policy CVaR objective at $\bhat_k$. Summing this inequality
and combining the preceding display with Lemma~G.4 and \emph{Azuma 1} gives
\begin{align*}
  \Reg^{\mathrm{RL}}_\tau(K)
  \le{}&6\mathrm e\,\tau^{-1}HL
  +2\mathrm e\,\tau^{-1}\!\left(
    27SAHK^{1/4}L^2+34S^2AHL^3
    +3\sqrt L\,\mathfrak D_K
  \right)\\
  \le{}&6\mathrm e\,\tau^{-1}\sqrt L\,\mathfrak D_K
  +54\mathrm e\,\tau^{-1}SAHK^{1/4}L^2
  +74\mathrm e\,\tau^{-1}S^2AHL^3,
\end{align*}
where the last inequality uses
$6HL\le6S^2AHL^3$ for $S,A,L\ge1$. This proves
\cref{eq:regret-reduction}. We retain the slightly looser coefficient
$74\mathrm e$, rather than the $70\mathrm e$ displayed by
\citet[Appendix~G.4, p.~34]{wang2023}, so that the absorption is valid
uniformly for all $S,A,L\ge1$; the coefficient $70\mathrm e$ would require
$S^2AL^2\ge3$.
None of these steps invokes Assumption~5.4.
In the original proof, that assumption enters only when the second bound on
Eq.~(19) controls the trajectory-variance sum in Eq.~(20). The selected-budget
self-bound in \cref{sec:self-bound} controls the same sum without the
assumption.

\section{Proof of the OCE-reduction extension}
\label{app:oce}

\begin{proof}[Proof of Proposition~\ref{prop:oce}]
Let $V_k^{\pi^k}(b)$ denote the initial value of the selected augmented
policy. For the CVaR utility
\[
  u_\tau(x):=\min\{x/\tau,0\},
\]
its value at the selected budget is
\[
  V_k^{\pi^k}(\widehat b_k)
  =\E[u_\tau(Z_k-\widehat b_k)\mid\cF_k]
  =-\tau^{-1}\E[Y_k^{\mathrm{OCE}}\mid\cF_k].
\]
Set
\[
  \Delta_k
  :=\widehat V_k(\widehat b_k)-V_k^{\pi^k}(\widehat b_k).
\]
At budget $0$, nonnegative returns give
$V_{\mathrm{aug}}^*(s_1,0)=0$. Exact optimism and the stated budget rule
therefore imply
\[
  \widehat b_k+\widehat V_k(\widehat b_k)
  \ge \widehat V_k(0)\ge0.
\]
Only comparison with the feasible budget $0$ is used here.
Since $0\le\widehat b_k\le1$,
\[
  \E[Y_k^{\mathrm{OCE}}\mid\cF_k]
  =\tau\bigl(\Delta_k-\widehat V_k(\widehat b_k)\bigr)
  \le\tau(1+\Delta_k).
\]
Moreover $0\le Y_k^{\mathrm{OCE}}\le1$, so its conditional variance is at
most its conditional mean. The bounded-regret condition in Definition~3.1
of \citet{wang2025}, with $V_u^{\max}=\tau^{-1}$, gives
$\sum_k\Delta_k\le\tau^{-1}\Reg_{\mathrm{Opt}}(K)$. Consequently,
\begin{equation}
  \sum_{k=1}^K\Var(Y_k^{\mathrm{OCE}}\mid\cF_k)
  \le K\tau+\Reg_{\mathrm{Opt}}(K).
  \label{eq:oce-self-bound}
\end{equation}

Writing $R=\Reg_{\mathrm{Opt}}(K)$, Assumption~D.10 and
\cref{eq:oce-self-bound} yield
\[
  R\le \sqrt{C_1(K\tau+R)}+C_2
  \le \sqrt{C_1K\tau}+\frac{R+C_1}{2}+C_2.
\]
Thus $R\le2\sqrt{C_1K\tau}+C_1+2C_2$. Under exact optimism, the
optimism-slack term in Theorem~3.2 of \citet{wang2025} vanishes. Since
$V_u^{\max}=\tau^{-1}$ for CVaR, that theorem gives
$\Reg_{\mathrm{CVaR}_\tau}(K)\le\tau^{-1}R$, which is
\cref{eq:oce-extension}.
\end{proof}

\bibliographystyle{plainnat}
\bibliography{references}

\end{document}